%% file: main.tex
\documentclass{article} 
\usepackage{iclr2027_conference,times}

\input{math_commands.tex}

\usepackage{amsthm}

\newtheorem{proposition}{Proposition}
\newtheorem{theorem}{Theorem}
\theoremstyle{definition}
\newtheorem{definition}{Definition}

\usepackage{hyperref}
\usepackage{url}
\usepackage{booktabs}
\usepackage{multirow}
\usepackage{colortbl}
\usepackage{arydshln}
\usepackage{graphicx}
\usepackage{wrapfig}
\usepackage{enumitem}

\title{Where to Adapt Matters: Layer-Selective Fine-Tuning for Capability Retention}

\author{Zhiqiang Pang\thanks{Equal contribution.}, \;
Zihong Sun\footnotemark[1], \;
Qi Xie, \;
Jun Shu, \;
Deyu Meng, \;
Zongben Xu \\
School of Mathematics and Statistics, Xi'an Jiaotong University, Xi'an, Shaanxi, China \\
\texttt{xjtupzq@gmail.com, szhc0gk@stu.xjtu.edu.cn} \\
\texttt{\{xie.qi, junshu, dymeng, zbxu\}@mail.xjtu.edu.cn}
}

\iclrfinalcopy 
\begin{document}

\maketitle
\lhead{Under review as a conference paper at ICLR 2027}

\begin{abstract}
Parameter-efficient fine-tuning (PEFT) enables large language models (LLMs) 
to adapt to specialized tasks, but often at the cost of degrading general capabilities 
acquired during pretraining. Existing approaches primarily mitigate this trade-off 
through data replay or regularization, relying on additional data or explicit optimization constraints. 
We instead focus on a different question: where should adaptation be applied?
We find that fine-tuning different Transformer layers produces different target-task gains 
and degrees of capability degradation, suggesting that not all layers are equally suitable for adaptation. 
To characterize this difference, we use layer-wise empirical Fisher information to measure target-task sensitivity. 
However, computing Fisher scores requires backward computation and becomes increasingly expensive for large models. 
We therefore introduce input–output cosine similarity as a lightweight, forward-only proxy for ranking layer sensitivity. 
Across models and tasks, layers with lower input–output similarity consistently exhibit higher empirical Fisher scores. 
Building on this observation, we propose Layer-Selective LoRA (LS-LoRA), which places trainable LoRA adapters only 
in layers with low input–output similarity. Experiments on mathematical reasoning and code generation 
show that LS-LoRA improves average target-task performance while retaining substantially more commonsense reasoning capability 
than standard all-layer LoRA, demonstrating that carefully choosing where to adapt 
can provide a simple and effective way to balance target-task adaptation and general capability retention.
\end{abstract}

\section{Introduction}

Large language models (LLMs) acquire strong general capabilities through
large-scale pretraining, enabling them to perform a wide range of downstream
tasks~\citep{brown2020language,bommasani2021opportunities}.
Despite the general capabilities, specialized applications such as mathematical
reasoning~\citep{xin2024deepseekprover} and code
generation~\citep{luo2024wizardcoder} often require fine-tuning on domain-specific
data to meet the demands of target tasks.
However, as LLMs continue to scale, full fine-tuning, which updates all model
parameters, incurs increasingly high computational and memory
costs~\citep{hu2022lora}.

Parameter-efficient fine-tuning (PEFT), particularly low-rank adaptation
(LoRA), reduces these costs by optimizing only a small number of parameters,
which is widely adopted in LLM
fine-tuning~\citep{houlsby2019parameter,li-liang-2021-prefix,lester2021power,hu2022lora}.
However, improvements in target-task performance through PEFT 
may come at the expense of general capabilities acquired during pretraining, 
a phenomenon known as forgetting~\citep{biderman2024lora}.
Therefore, effective adaptation requires improving target-task performance
while retaining these general capabilities~\citep{huang2026peftarena}.

Existing approaches mitigate capability degradation primarily through data
replay and regularization. Data replay mixes examples from
general domains into downstream training to retain general capabilities
~\citep{scialom-etal-2022-fine,bethune2025scaling}.
However, replay methods that rely on stored examples may be limited by data
availability, particularly when the original pretraining corpus is
inaccessible~\citep{marek2026forgetting}.
Regularization methods preserve prior knowledge by constraining parameter
changes~\citep{kirkpatrick2017overcoming,li2018explicit}.
However, these constraints may limit improvements in target-task performance~\citep{zenke2017continual}.
These methods mainly rely on the data or constraints used during fine-tuning,
leaving a different question underexplored: \textit{where should adaptation be applied?}

\begin{wrapfigure}{r}{0.4\textwidth}
    \vspace{-2mm}
    \centering
    \includegraphics[width=\linewidth]{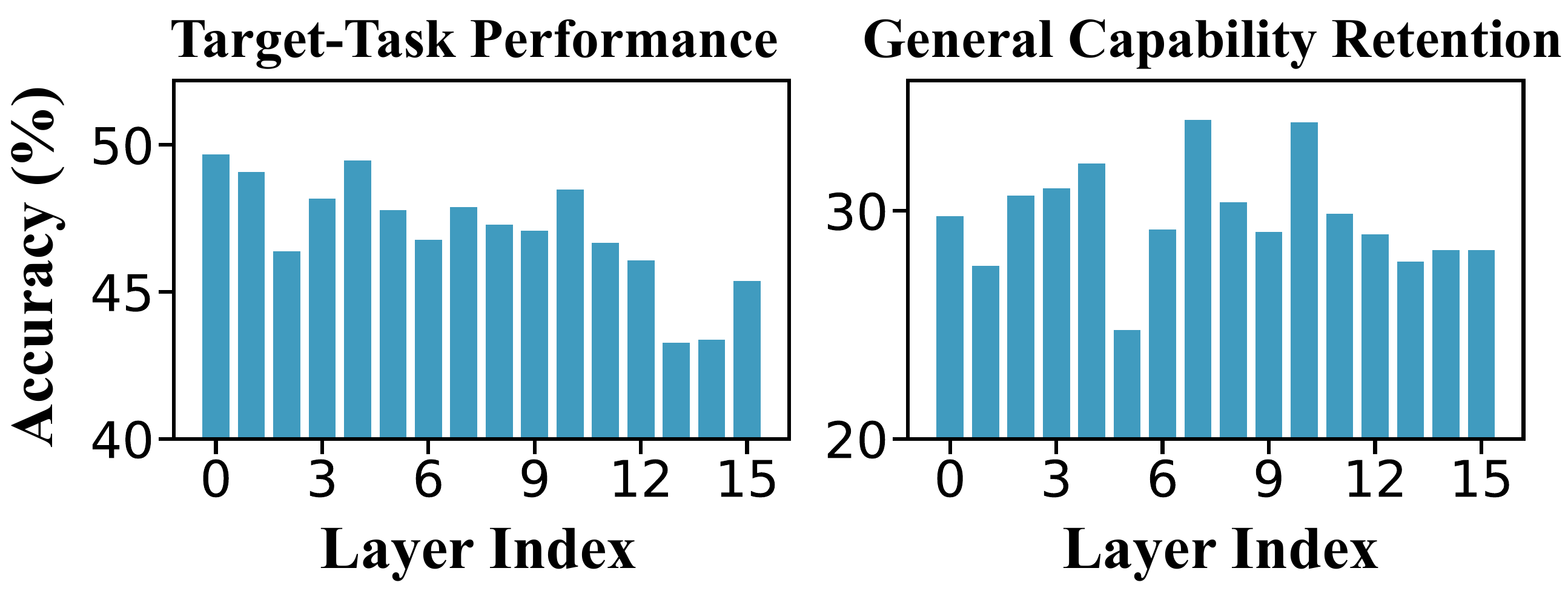}
    \caption{Effects of independently fine-tuning each layer of Llama-3.2-1B on Math-10K. 
    Each bar corresponds to fine-tuning a single layer. 
    Left: target-task performance. 
    Right: general capability retention after fine-tuning.}
    \vspace{-2mm}
    \label{Figs:layerwise}
\end{wrapfigure}


We find that the choice of layers to adapt matters for both target-task performance and general capability retention.
Standard LoRA typically uses the same adaptation configuration across Transformer layers, 
yet fine-tuning different layers can lead to different effects on adaptation and retention, as shown in Fig.~\ref{Figs:layerwise}. 
Adapting some layers contribute more to target-task performance, while adapting others lead to greater degradation of general capabilities. 
These observations motivate selecting which layers to adapt.

The next question is how to identify which layers should be adapted for a given target task. 
We use layer-wise empirical Fisher scores to characterize target-task sensitivity, 
and introduce input–output cosine similarity as an efficient forward-only proxy for this sensitivity. 
Based on this proxy, we propose Layer-Selective LoRA (LS-LoRA), which trains LoRA adapters only
in layers with low input–output cosine similarity.

Our main contributions are as follows:
\begin{itemize}
    \item We use layer-wise empirical Fisher scores to characterize 
    differences in target-task sensitivity across LLM layers, 
    motivating a layer-selective adaptation strategy.
    \item We propose input–output cosine similarity as an efficient proxy
    for layer-wise sensitivity. It requires only a single forward pass 
    over target-task calibration inputs, without response labels or backpropagation.
    Our analysis shows consistent negative rank correlations between this
    proxy and empirical Fisher scores across models and tasks.
    \item We introduce LS-LoRA, which trains 
    LoRA adapters only in layers with low input–output cosine
    similarity. 
    Experiments on mathematical reasoning and code generation 
    show that LS-LoRA improves average target-task performance 
    while reducing degradation of the evaluated general capabilities.
\end{itemize}

\section{Related Work}

\subsection{Parameter-Efficient Fine-Tuning}

PEFT adapts pretrained models through lightweight modules or restricted updates,
including adapters~\citep{houlsby2019parameter}, prefix and prompt
tuning~\citep{li-liang-2021-prefix,lester2021power}, and low-rank updates in
LoRA~\citep{hu2022lora}. AdaLoRA further allocates rank budgets according to
importance scores estimated during training~\citep{zhang2023adaptive}.
Our work complements these approaches by selecting which Transformer layers
receive LoRA adapters before fine-tuning.

Parameter efficiency alone does not prevent capability degradation
\citep{biderman2024lora}. Replay methods retain knowledge using additional
training examples~\citep{scialom-etal-2022-fine}, while regularization methods
constrain parameter updates~\citep{kirkpatrick2017overcoming,li2018explicit}.
In particular, EWC uses Fisher information to protect parameters important
for previous tasks~\citep{kirkpatrick2017overcoming}. We instead use empirical
Fisher scores to analyze target-task layer sensitivity and investigate layer
selection as a way to balance adaptation and capability retention.

\subsection{Layer-wise Analysis and Selective Adaptation}

Transformer layers differ in their representations and functional
importance~\citep{voita2019bottomup,zhang2024investigating}.
ShortGPT uses input--output similarity to identify redundant layers for
pruning~\citep{men2025shortgpt}. For adaptation, layer-freezing studies examine
predefined subsets~\citep{hwang2025exploring}, while LISA stochastically activates
subsets during training~\citep{pan2024lisa}.
Representation-based approaches also guide layer selection.
\citet{liu-etal-2025-spectral} use Centered Kernel Alignment to identify
intrinsic critical layers, and \citet{ogawa2026layerwise} study
similarity-based layer-wise LoRA fine-tuning.
Layer Card uses task-dependent residual, optimization, and cost diagnostics
to guide PEFT module placement~\citep{xu2026understanding}.
Our contribution is to connect a forward-only layer-selection criterion
to target-task empirical Fisher sensitivity and evaluate its consequences
for both adaptation and capability retention.
We use input--output cosine similarity to select a fixed set of layers
before fine-tuning, without response labels or backpropagation during
selection.

\section{Method}
\label{sec:method}

We first introduce Fisher information as a measure of parameter sensitivity,
then use it to examine differences across LLM layers. These differences motivate
our layer-selection method, which uses input–output cosine similarity to choose
layers for fine-tuning with only forward computation.

\subsection{Preliminaries}
\label{sec:Preliminaries}

\begin{definition}[Fisher information~\citep{fisher1922foundations}]
\label{def:fisher_information}
Let $Z\sim p_\theta$ have a probability density or mass function,
with $\theta$ in an open set $\Theta\subseteq\mathbb{R}^d$.
Assume that $\log p_\theta(Z)$ is differentiable in $\theta$
almost surely and that its gradient has a finite second moment.
The Fisher information, which quantifies the information
that $Z$ carries about $\theta$, is defined as
\begin{equation}
    F(\theta)
    =\mathbb{E}_{Z\sim p_\theta}
    \left[
    \nabla_\theta\log p_\theta(Z)
    \left(\nabla_\theta\log p_\theta(Z)\right)^\top
    \right].
    \label{eq:general_fisher}
\end{equation}
\end{definition}

The following proposition relates Fisher information to the local sensitivity
of the model distribution to parameter changes. 
The proof is provided in Appendix~\ref{app:proof_fisher_kl}.

\begin{proposition}[Fisher information and local distributional sensitivity]
\label{pro:fisher_kl}
Let $F(\theta)$ be the Fisher information in
Definition~\ref{def:fisher_information}. For fixed $\theta$, assume that,
for $\theta+\delta$ in a neighborhood of $\theta$, the distributions $p_{\theta+\delta}$
have a common support independent of $\delta$,
that $p_{\theta+\delta}(z)$ is twice continuously differentiable in $\delta$,
and that $D_{\mathrm{KL}}(p_\theta\Vert p_{\theta+\delta})$ is finite and twice
continuously differentiable in $\delta$. Assume also that derivatives with
respect to $\delta$ up to second order may be passed under the integrals
(or sums) defining normalization
and this KL divergence. Then, as $\delta\to0$ with
$\theta+\delta\in\Theta$,
\begin{equation}
    D_{\mathrm{KL}}\!\left(p_\theta\,\Vert\,p_{\theta+\delta}\right)
    =\frac12\delta^\top F(\theta)\delta+o(\|\delta\|_2^2).
    \label{eq:fisher_kl}
\end{equation}
\end{proposition}

Proposition~\ref{pro:fisher_kl} shows that, for small parameter
perturbations $\delta$ of equal magnitude, higher Fisher information
$F(\theta)$ along the perturbation direction induces a larger shift
in the model distribution, as measured by KL divergence.

For large neural networks, exact computation of the Fisher information
matrix is typically impractical. Evaluating the expectation over the model
distribution can be expensive, and explicitly storing the full matrix
requires memory quadratic in the number of parameters~\citep{martens2020new}.
For our layer-wise analysis, we compute only the diagonal entries of the
empirical Fisher from observed samples $\{z_i\}_{i=1}^N$:
\begin{equation}
    \widehat F_{jj}(\theta)
    =\frac1N\sum_{i=1}^{N}
    \left(\frac{\partial\log p_\theta(z_i)}{\partial\theta_j}\right)^2,
    \qquad j=1,\ldots,d.
    \label{eq:empirical_fisher}
\end{equation}
These entries can be accumulated from squared gradients without forming
the full matrix. They provide a practical measure of the sensitivity of
the observed-data log-likelihood to individual parameters.

\subsection{Fisher-based Layer Sensitivity in LLM Fine-tuning}
\label{sec:fisher_analysis}

As shown in Proposition~\ref{pro:fisher_kl}, Fisher information quantifies how strongly
small parameter perturbations affect the model distribution.
Motivated by this sensitivity interpretation,
we use the empirical Fisher to compare how sensitive the target-task
log-likelihood is to parameter changes in different LLM layers.
We define a layer as one complete Transformer block, including its attention,
feed-forward, normalization, and residual operations. 
Consider an LLM with $L$ layers. Each layer $l$ maps
its input representation $H^{(l-1)}$ to an output representation
$H^{(l)}$:
\begin{equation}
    H^{(l)}=f_l\left(H^{(l-1)};\theta^{(l)}\right),
    \qquad l=1,\ldots,L,
    \label{eq:transformer_block}
\end{equation}
where $f_l$ denotes the transformation performed by layer $l$,
parameterized by $\theta^{(l)}$. 
Here, $H^{(l)}\in\mathbb{R}^{T\times d_h}$ contains the hidden
representations of all tokens after layer $l$, where $T$ is the
sequence length and $d_h$ is the hidden dimension.

To measure layer sensitivity on a target task $\tau$, we use a dataset
$\mathcal{S}_\tau=\{(x_i,y_i)\}_{i=1}^N$, where $x_i$ is an input
and $y_i$ is its target response. The conditional log-likelihood
$\log p_\theta(y_i\mid x_i)$ sums the token log-probabilities in
the response. 
We compute $\widehat F_{jj}(\theta)$ using Eq.~(\ref{eq:empirical_fisher})
with the conditional log-likelihood $\log p_\theta(y_i\mid x_i)$ on
$\mathcal{S}_\tau$. Averaging these entries over the parameters in layer $l$
gives the scalar sensitivity score:
\begin{equation}
    I_l
    =\frac{1}{d_l}\sum_{j\in\mathcal{J}_l}\widehat F_{jj}(\theta)
    =\frac{1}{N d_l}\sum_{i=1}^{N}
    \sum_{j\in\mathcal{J}_l}
    \left(\frac{\partial\log p_\theta(y_i\mid x_i)}
    {\partial\theta_j}\right)^2,
    \label{eq:block_fisher_score}
\end{equation}
where $\mathcal{J}_l$ is the set of parameter indices belonging to layer $l$
and $d_l=|\mathcal{J}_l|$ is its number of parameters.
A larger $I_l$ indicates greater sensitivity of
the target-response log-likelihood to changes in that layer's
parameters. The score can be computed directly by averaging squared
partial derivatives without explicitly forming the matrix.

\begin{figure}[t]
    \vspace{-0mm}
    \hspace{-0mm}\includegraphics[width=0.95\linewidth]{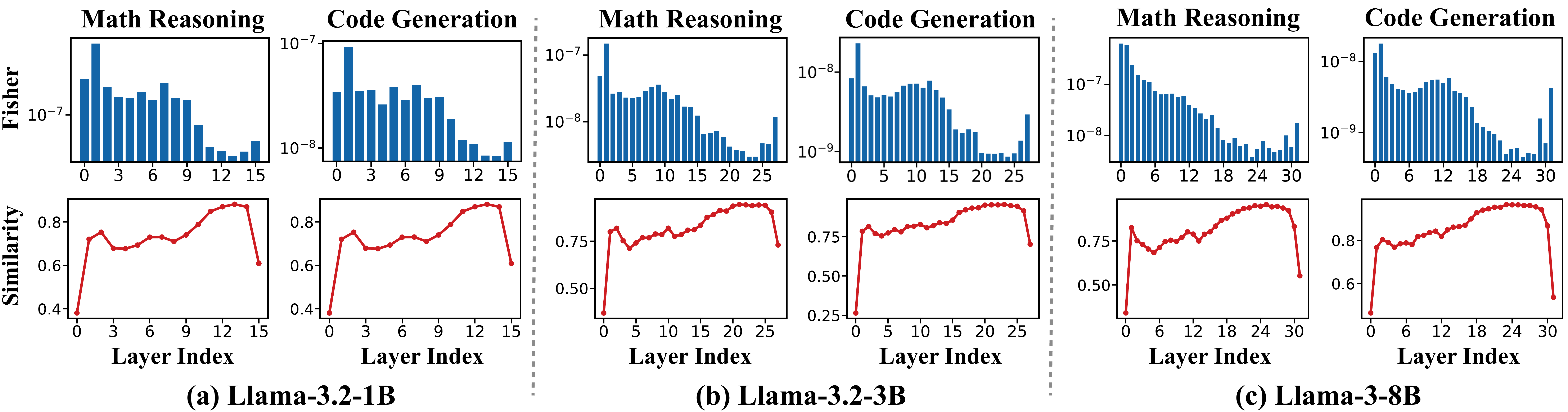}\centering
    \vspace{-2mm}
    \caption{Layer-wise empirical Fisher scores (top)
    and cosine similarities between layer input and output representations (bottom) for (a) Llama-3.2-1B, 
    (b) Llama-3.2-3B, and (c) Llama-3-8B on mathematical reasoning and code generation.}
    \label{Figs:fisher_cosine}
    \vspace{-0mm}
\end{figure}

As shown in Fig.~\ref{Figs:fisher_cosine}, empirical Fisher scores
vary substantially across layers for both Llama-3.2-1B and
Llama-3.2-3B on mathematical reasoning and code generation.
Scores are generally higher in shallow layers, decrease toward
deeper layers, and rise again in the final layers. 
These observations motivate considering
target-task layer sensitivity when selecting layers for adaptation.

\subsection{Layer Selection via input–output Similarity}
\label{sec:layer_selection}

\begin{wrapfigure}{r}{0.25\textwidth}
    \vspace{-2mm}
    \centering
    \includegraphics[width=\linewidth]{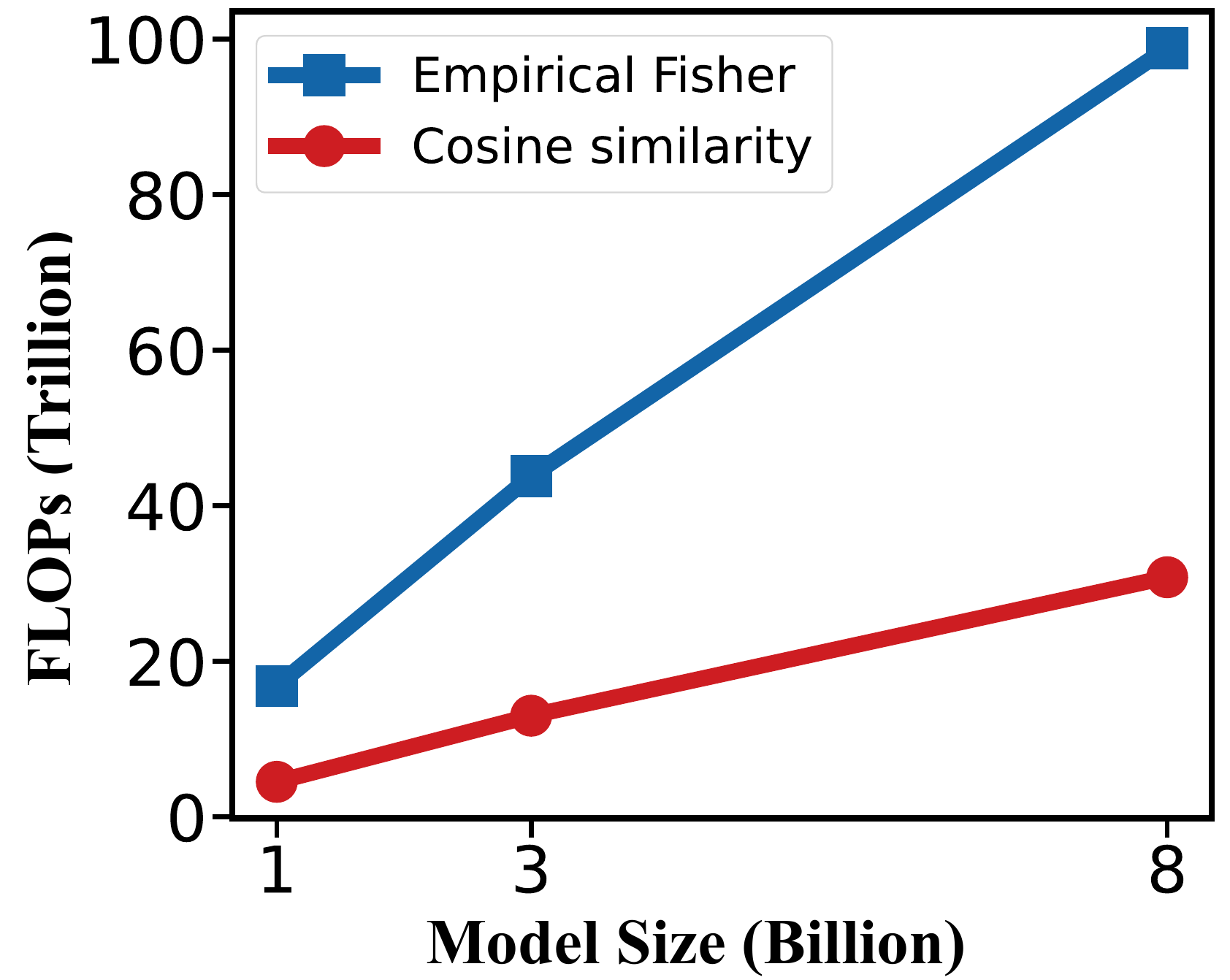}
    \caption{FLOPs for computing empirical Fisher and cosine similarity
    by model size.}
    \vspace{-2mm}
    \label{Figs:flops}
\end{wrapfigure}

Although empirical Fisher scores characterize layer-wise sensitivity,
computing them requires backpropagation, which incurs substantial memory
and computation costs for LLMs. As shown in
Fig.~\ref{Figs:flops}, the FLOPs required to compute Fisher scores
increase rapidly with model size.

Intuitively, a higher empirical Fisher score indicates that a layer
is more sensitive to parameter changes on the target task.
This sensitivity may be reflected in how the layer transforms
its input representations. Layers with lower input--output cosine
similarity change representation directions more substantially,
suggesting a more active role in processing task-relevant information.
Such layers may therefore exhibit greater task sensitivity and
higher empirical Fisher scores.
Motivated by this intuition, we use
input–output cosine similarity as a proxy for layer-wise sensitivity
in fine-tuning. Computing this similarity requires only a single forward
pass, enabling efficient comparison across layers without backpropagation, as shown in
Fig.~\ref{Figs:flops}.

Given a calibration set of target-task inputs
$\mathcal{U}_\tau=\{x_i\}_{i=1}^{N_c}$, we run the pretrained model
to obtain each layer's input and output representations.
For input $x_i$, let $H_i^{(l)}$ denote the representation matrix
defined in Eq.~\ref{eq:transformer_block}, and let
$h_{i,t}^{(l)}\in\mathbb{R}^{d_h}$
denote the representations of the $t$-th token in $H_i^{(l)}$. 
We compute the similarity score of layer $l$ as
\begin{equation}
    s_l
    =\frac{1}{|\mathcal{V}|}
    \sum_{(i,t)\in\mathcal{V}}
    \frac{
        \left\langle h_{i,t}^{(l-1)},h_{i,t}^{(l)}\right\rangle
    }{
        \left\|h_{i,t}^{(l-1)}\right\|_2
        \left\|h_{i,t}^{(l)}\right\|_2+\epsilon
    },
    \label{eq:layer_cosine_similarity}
\end{equation}
where $\mathcal{V}$ specifies the token positions used for scoring,
and $\epsilon>0$ ensures numerical stability.

By default, we use only the last non-padding token of each input,
so that $\mathcal{V}=\{(i,T_i)\}_{i=1}^{N_c}$, where $T_i$ denotes
the final valid token position in $x_i$. Thus, $s_l$ averages
the last-token input–output similarity across calibration inputs.
This gives one score per layer without response labels or gradients.

The following theorem gives a theoretical motivation for this proxy
under local parameter and average task-sensitivity conditions.

\begin{theorem}[Cosine similarity and an empirical Fisher lower bound]
\label{thm:cosine_fisher_bound}
Consider a differentiable residual layer with nonzero parameters
$\theta^{(l)}$ and nonzero selected input and output representations.
Assume that the task log-likelihood is sensitive to the full residual
directions on average, and that a bounded parameter direction
approximately reproduces their effect on the task to first order.
Then the empirical Fisher score $I_l$ and input--output similarity $s_l$,
computed on the same examples with $\epsilon=0$, satisfy
\begin{equation}
    I_l\geq A_l(1-s_l),
    \qquad
    A_l=\frac{2(1-\eta_l)^2\kappa_l^2m_l}
    {d_l C_l^2\|\theta^{(l)}\|_2^2}>0.
    \label{eq:cosine_fisher_bound}
\end{equation}
Here $d_l$ is the number of layer parameters, $m_l>0$ is the minimum
product of the selected input and output norms, and $\kappa_l>0$
lower-bounds the average directional task sensitivity.
The parameter direction $v_l$ satisfies
$\|v_l\|_2\leq C_l\|\theta^{(l)}\|_2$ with $C_l>0$, while
$\eta_l\in[0,1)$ bounds its relative approximation error measured through
the task gradients. Full definitions and the proof are in
Appendix~\ref{app:proof_cosine_fisher}.
\end{theorem}

Theorem~\ref{thm:cosine_fisher_bound} establishes that, for fixed $A_l$,
lower input--output similarity implies a higher lower bound on the empirical
Fisher score. Consistent with this intuition,
Fig.~\ref{Figs:fisher_cosine} qualitatively show that layers with lower
similarity tend to have higher empirical Fisher scores.
To further quantify this relationship, we compute \textit{Spearman's rank correlation}~\citep{spearman1904proof}
between layer-wise empirical Fisher scores and input--output cosine
similarities, treating each layer as one paired observation.
Negative correlations indicate that lower-similarity layers tend to have
higher empirical Fisher scores. The formal definition and computation
are provided in Appendix~\ref{app:spearman}.

\begin{figure}[t]
    \vspace{-0mm}
    \hspace{-0mm}\includegraphics[width=0.8\linewidth]{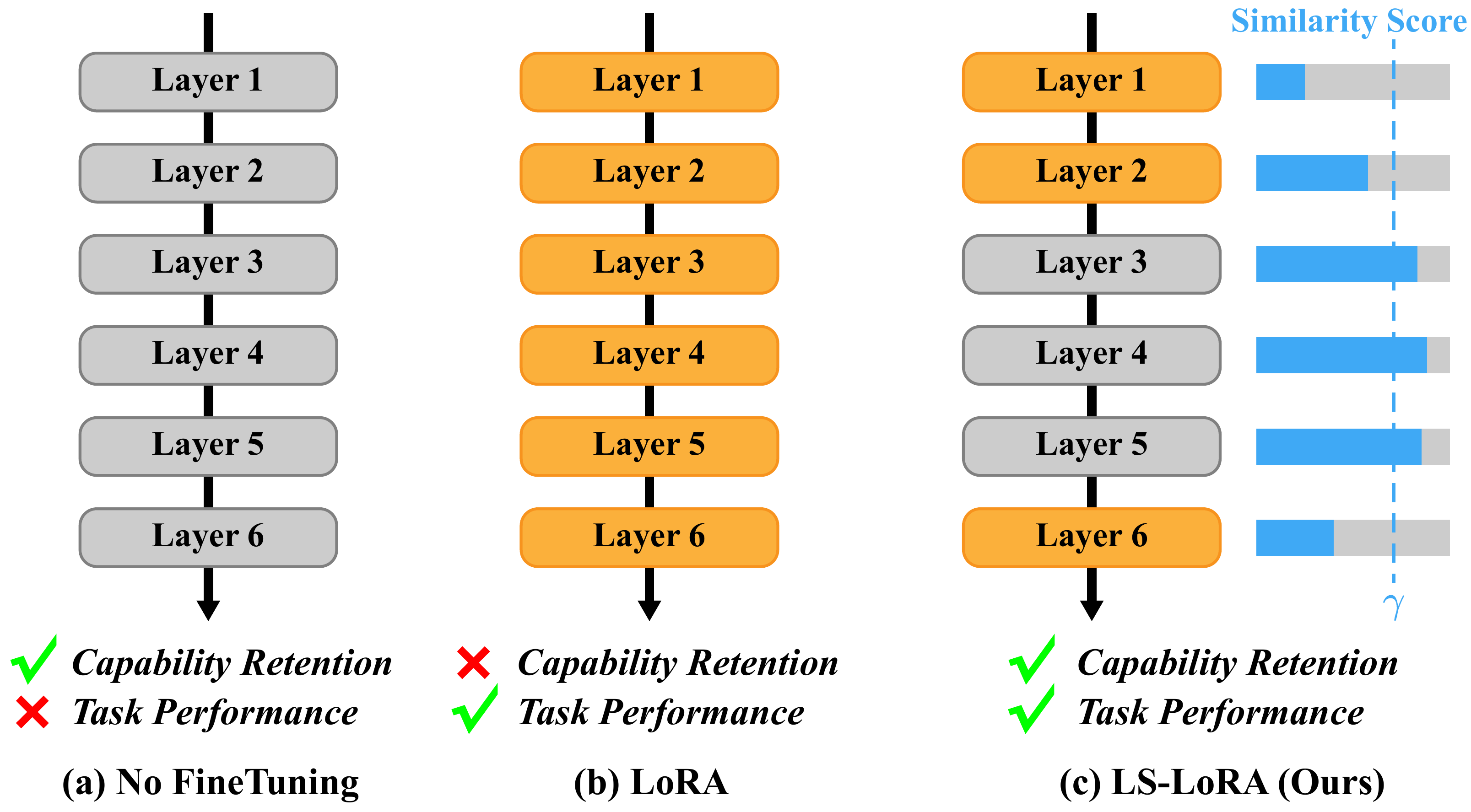}\centering
    \vspace{-2mm}
    \caption{Overview of LS-LoRA.
    (a) The pretrained model without fine-tuning.
    (b) Standard LoRA applies adapters to all layers.
    (c) LS-LoRA applies adapters only to layers whose input–output cosine
    similarity scores fall below the threshold $\gamma$ (dashed line),
    aiming to improve target-task performance while retaining pretrained
    capabilities. Orange denotes layers with trainable adapters, and gray
    denotes layers without adaptation.}
    \label{Figs:method}
    \vspace{-0mm}
\end{figure}

\begin{table}[h]
\caption{Spearman's rank correlations between layer-wise
input–output cosine similarity and empirical Fisher scores
for Llama-3.2-1B, Llama-3.2-3B, Llama-3-8B, and Mistral-7B on
mathematical reasoning and code generation.
Values in parentheses are the corresponding $p$-values.}
\label{Table:Spearman}
\centering 
\setlength{\tabcolsep}{7.5pt}
{
  \begin{tabular}{lcccc}
    \toprule[1.1pt]
    Task & Llama-3.2-1B & Llama-3.2-3B & Llama-3-8B & Mistral-7B \\
    \midrule
    Mathematical Reasoning & -0.82 (1e-8) & -0.84 (2e-9) & -0.89 (6e-12) & -0.85 (7e-10) \\
    Code Generation        & -0.76 (5e-7) & -0.82 (8e-9) & -0.86 (3e-10) & -0.91 (9e-13) \\
    \bottomrule[1.1pt]
  \end{tabular}
  }
\end{table}

As shown in Table~\ref{Table:Spearman}, we compute Spearman's
rank correlations between input–output cosine similarity
and empirical Fisher scores across layers for
Llama-3.2-1B, Llama-3.2-3B, Llama-3-8B, and Mistral-7B on mathematical
reasoning and code generation. The correlations range from $-0.76$ to
$-0.91$, indicating that layers with lower cosine
similarity tend to have higher empirical Fisher scores.
We also compute the corresponding $p$-values, reported in parentheses
in Table~\ref{Table:Spearman}. 
All reported $p$-values are below $10^{-5}$, indicating statistically
significant negative correlations across all evaluated model--task pairs.
Together, the strong negative correlations and their statistical
significance provide compelling statistical evidence that lower
input–output similarity is associated with higher empirical Fisher scores
across the evaluated models and tasks. These results substantiate the use
of input–output similarity as a practical proxy for ranking layers by
target-task sensitivity.

As illustrated in Fig.~\ref{Figs:method}, we select layer $l$ for fine-tuning
if its input–output cosine similarity
score satisfies $s_l<\gamma$, where $\gamma$ is a predefined threshold.
We place trainable LoRA adapters only in the selected layers while keeping
all pretrained weights frozen.
The selected layers are fixed before fine-tuning, and only their adapter
parameters are trained using the target-task supervised loss.
We call this method Layer-Selective LoRA (LS-LoRA).

\section{Experiments}


We first evaluate LS-LoRA on mathematical reasoning and code generation, 
using commonsense reasoning benchmarks to assess capability retention. 
We then conduct ablation studies to assess key design choices.

\subsection{Experimental Setup}

\paragraph{Tasks and evaluation.}
For mathematical reasoning, we fine-tune models on MATH-10K~\citep{hu-etal-2023-llm} and evaluate
accuracy on GSM8K, AQuA, and SVAMP. For code generation, we fine-tune on
Magicoder-Evol and evaluate on HumanEval+, reporting pass@$k$ for
$k\in\{1,5,10\}$. 
To assess retention of non-target capabilities, we evaluate commonsense
reasoning accuracy on ARC-Challenge (ARC-C), OpenBookQA (OB-QA), and
Social IQA (SI-QA). These benchmarks provide a focused assessment of
commonsense reasoning rather than a comprehensive evaluation of general
capabilities. We report individual benchmark scores
and their mean for mathematical reasoning and commonsense reasoning; for code
generation, we additionally report the mean of the three pass@$k$ scores.

\paragraph{Baselines.}
We include the model before fine-tuning, denoted as \textit{Base}, as a reference
for both target-task performance and general capability retention. We compare our
method with standard \textit{LoRA}~\citep{hu2022lora}, which trains adapters in all Transformer
layers, and two regularization-based baselines for mitigating forgetting:
L2-SP~\citep{li2018explicit} and EWC~\citep{kirkpatrick2017overcoming}.
Both regularization methods are combined with LoRA fine-tuning,
yielding \textit{L2-SP-LoRA} and \textit{EWC-LoRA}, respectively.


\paragraph{Models.}
We evaluate four models from two families: Llama-3.2-1B, Llama-3.2-3B,
Llama-3-8B, and Mistral-7B. Table~\ref{Table:math_reasoning} reports
mathematical reasoning results for all four models, and
Table~\ref{Table:code_generation} reports their code generation results.


\paragraph{Layer selection and training.}
Before fine-tuning, we randomly sample 100 examples from the target-task
training set and use their inputs as the calibration set. A single forward
pass over this set yields the input and output representations of all
Transformer layers. Following Section~\ref{sec:layer_selection}, we compute
each layer's cosine similarity at the last non-padding input token and
average it across the calibration examples. We set the selection threshold
to $\gamma=0.85$ for all experiments and train LoRA adapters
only in layers with $s_l<\gamma$. The selected layers remain fixed throughout
fine-tuning, and all pretrained weights remain frozen. This selection step
requires neither response labels nor backpropagation.
Following~\citet{pmlr-v267-liu25ba}, we train all fine-tuning methods with AdamW 
for three epochs, using a LoRA rank of 128 for mathematical reasoning and 64 for code generation.

\begin{table}[t]\small
\vspace{-2mm}
\caption{Mathematical reasoning performance and general capability retention after fine-tuning on mathematical reasoning tasks.
Avg. denotes the category mean; \textbf{bold} marks the best reported average among
fine-tuned methods for each model. Higher is better.
}
\label{Table:math_reasoning}
\centering 
\setlength{\tabcolsep}{4.8pt}
{
  \begin{tabular}{llcccccccc}
    \toprule[1.1pt]
    \multirow[c]{2}{*}[-0.5ex]{Model}& \multirow[c]{2}{*}[-0.5ex]{Method} & \multicolumn{4}{c}{Target-Task Performance (Math)} & \multicolumn{4}{c}{General Capability Retention}\\
    \cmidrule(lr){3-6}\cmidrule(lr){7-10}
    &&GSM8K&AQuA&SVAMP&Avg.&ARC-C&OB-QA&SI-QA&Avg. \\
    \midrule
    \multirow{5}{*}{Llama-3.2-1B} 
    & Base        & 31.0 & 6.7  & 60.1 & 32.6 & 28.8 & 30.6 & 40.8 & 33.4 \\
    \addlinespace[0.1em]
    \cdashline{2-10}
    \addlinespace[0.2em]
    & LoRA        & 49.5 & 38.2 & 67.6 & 51.8 & 22.8 & 24.6 & 29.9 & 25.8 \\
    & L2-SP-LoRA  & 42.8 & 25.6 & 70.3 & 46.2 & 26.1 & 29.2 & 34.0 & 29.8 \\
    & EWC-LoRA    & 50.3 & 36.6 & 68.6 & 51.9 & 22.8 & 27.4 & 32.3 & 27.5 \\
\rowcolor{cyan!10} \cellcolor{white} & LS-LoRA   & 50.3 & 37.8 & 68.4 & \textbf{52.2} & 30.0 & 32.2 & 33.1 & \textbf{31.8} \\

    \midrule
    \multirow{5}{*}{Llama-3.2-3B} 
    & Base   & 62.0 & 19.3 & 83.7 & 55.0 & 62.2 & 56.4 & 58.8 & 59.1 \\
    \addlinespace[0.1em]
    \cdashline{2-10}
    \addlinespace[0.2em]
    & LoRA   & 69.6 & 47.6 & 80.3 & 65.8 & 44.1 & 52.4 & 34.9 & 43.8 \\
    & L2-SP-LoRA  & 70.1 & 39.0 & 86.5 & 65.2 & 52.2 & 57.6 & 47.7 & 52.5 \\
    & EWC-LoRA    & 68.0 & 48.8 & 78.7 & 65.2 & 51.5 & 58.4 & 42.7 & 50.9 \\
\rowcolor{cyan!10} \cellcolor{white} &LS-LoRA   & 69.5 & 48.8 & 82.3 & \textbf{66.9} & 59.3 & 62.4 & 52.6 & \textbf{58.1} \\

    \midrule
    \multirow{5}{*}{Llama-3-8B} 
    & Base   & 59.6 & 49.6 & 53.8 & 54.3 & 33.1 & 25.6 & 37.4 & 32.0 \\
    \addlinespace[0.1em]
    \cdashline{2-10}
    \addlinespace[0.2em]
    & LoRA   & 69.8 & 42.6 & 79.6 & 64.0 & 25.8 & 31.8 & 18.5 & 25.4 \\
    & L2-SP-LoRA  & 69.4 & 40.9 & 80.5 & 63.6 & 28.1 & 37.8 & 20.9 & 28.9 \\
    & EWC-LoRA    & 68.5 & 42.5 & 76.5 & 62.5 & 28.8 & 31.8 & 26.7 & 29.1 \\
\rowcolor{cyan!10} \cellcolor{white} & LS-LoRA   & 70.4 & 45.3 & 78.6 & \textbf{64.8} & 33.9 & 37.2 & 33.4 & \textbf{34.8} \\

    \midrule
    \multirow{5}{*}{Mistral-7B} 
    & Base   & 38.5 & 20.5 & 55.2 & 38.1 & 66.6 & 61.4 & 65.1 & 64.4 \\
    \addlinespace[0.1em]
    \cdashline{2-10}
    \addlinespace[0.2em]
    & LoRA   & 59.8 & 32.7 & 73.8 & 55.4 & 34.6 & 39.0 & 37.5 & 37.0 \\
    & L2-SP-LoRA  & 58.5 & 33.5 & 72.9 & 55.0 & 42.6 & 40.6 & 41.2 & 41.5 \\
    & EWC-LoRA    & 60.9 & 30.7 & 73.8 & 55.1 & 42.3 & 43.0 & 34.5 & 39.9 \\
\rowcolor{cyan!10} \cellcolor{white} & LS-LoRA   & 61.0 & 33.1 & 74.6 & \textbf{56.2} & 48.3 & 42.4 & 47.9 & \textbf{46.2} \\
    \bottomrule[1.1pt]
  \end{tabular}
  }
\end{table}

\begin{table}[t]\small
\vspace{-2mm}
\caption{Code generation performance and general capability retention after fine-tuning on code generation tasks.
Avg. denotes the mean of the three reported metrics in each category;
\textbf{bold} marks the best average among fine-tuned methods for each model.
Higher is better.
}
\label{Table:code_generation}
\centering 
\setlength{\tabcolsep}{4.6pt}
{
  \begin{tabular}{llcccccccc}
    \toprule[1.1pt]
    \multirow[c]{2}{*}[-0.5ex]{Model}& \multirow[c]{2}{*}[-0.5ex]{Method} & \multicolumn{4}{c}{Target-Task Performance (Code)} & \multicolumn{4}{c}{General Capability Retention}\\
    \cmidrule(lr){3-6}\cmidrule(lr){7-10}
    &&Pass@1&Pass@5&Pass@10&Avg.&ARC-C&OB-QA&SI-QA&Avg. \\
    \midrule
    \multirow{5}{*}{Llama-3.2-1B} 
    & Base   & 33.0 & 43.3  & 47.2 & 41.2 & 28.8 & 30.6 & 40.8 & 33.4 \\
    \addlinespace[0.1em]
    \cdashline{2-10}
    \addlinespace[0.2em]
    & LoRA   & 35.6 & 44.7 & 47.8 & 42.7 & 26.1 & 28.6 & 39.1 & 31.3 \\
    & L2-SP-LoRA  & 32.9 & 42.7 & 45.7 & 40.4 & 26.4 & 30.4 & 39.9 & 32.2 \\
    & EWC-LoRA    & 34.3 & 42.8 & 46.6 & 41.3 & 26.5 & 31.0 & 39.2 & 32.2 \\
\rowcolor{cyan!10} \cellcolor{white} & LS-LoRA   & 35.8 & 45.6 & 49.2 & \textbf{43.5} & 26.9 & 30.6 & 39.9 & \textbf{32.5} \\

    \midrule
    \multirow{5}{*}{Llama-3.2-3B} 
    & Base   & 50.6 & 60.8  & 62.1 & 57.8 & 62.2 & 56.4 & 58.8 & 59.1 \\
    \addlinespace[0.1em]
    \cdashline{2-10}
    \addlinespace[0.2em]
    & LoRA   & 51.5 & 61.1 & 63.9 & 58.8 & 68.5 & 62.2 & 63.3 & 64.7 \\
    & L2-SP-LoRA  & 49.2 & 62.5 & 66.4 & 59.4 & 66.5 & 60.6 & 62.1 & 63.1 \\
    & EWC-LoRA    & 51.6 & 61.0 & 64.0 & 58.9 & 68.5 & 62.6 & 63.1 & 64.7 \\
\rowcolor{cyan!10} \cellcolor{white} & LS-LoRA   & 52.6 & 62.1 & 64.3 & \textbf{59.6} & 68.9 & 65.2 & 65.7 & \textbf{66.6} \\

    \midrule
    \multirow{5}{*}{Llama-3-8B} 
    & Base   & 55.1 & 63.1 & 65.2 & 61.1 & 33.1 & 25.6 & 37.4 & 32.0 \\
    \addlinespace[0.1em]
    \cdashline{2-10}
    \addlinespace[0.2em]
    & LoRA   & 53.4 & 64.9 & 68.5 & 62.3 & 68.9 & 53.2 & 62.9 & 61.7 \\
    & L2-SP-LoRA  & 52.4 & 63.0 & 66.0 & 60.5 & 65.6 & 59.0 & 63.7 & 62.8 \\
    & EWC-LoRA    & 54.1 & 63.5 & 67.8 & 61.8 & 65.9 & 62.8 & 59.1 & 62.6 \\
 \rowcolor{cyan!10} \cellcolor{white} & LS-LoRA   & 53.2 & 65.2 & 68.8 & \textbf{62.4} & 72.6 & 63.8 & 65.7 & \textbf{67.4} \\

 \midrule
    \multirow{5}{*}{Mistral-7B} 
    & Base   & 34.51 & 44.3 & 48.7 & 42.5 & 66.6 & 61.4 & 65.1 & 64.4 \\
    \addlinespace[0.1em]
    \cdashline{2-10}
    \addlinespace[0.2em]
    & LoRA   & 38.5 & 49.5 & 53.2 & 47.0 & 41.1 & 31.4 & 41.1 & 37.9 \\
    & L2-SP-LoRA  & 38.2 & 47.2 & 50.6 & 45.4 & 42.5 & 35.1 & 43.2 & 40.3 \\
    & EWC-LoRA    & 38.5 & 48.2 & 51.8 & 46.2 & 45.1 & 31.0 & 44.3 & 40.1 \\
\rowcolor{cyan!10} \cellcolor{white} & LS-LoRA   & 40.3 & 51.1 & 54.5 & \textbf{48.6} & 48.9 & 39.2 & 44.1 & \textbf{44.1} \\

    \bottomrule[1.1pt]
  \end{tabular}
  }
\end{table}

\subsection{Mathematical Reasoning}

Table~\ref{Table:math_reasoning} shows that LS-LoRA achieves the highest
average mathematical and commonsense reasoning accuracy among the
compared fine-tuning methods on all four models. Relative to all-layer
LoRA, it improves these averages by 0.4--1.1 and 6.0--14.3 percentage
points, respectively, improving both target-task performance and
capability retention.
The largest retention gain occurs on Llama-3.2-3B: LS-LoRA narrows
the commonsense accuracy drop from 15.3 to 1.0 percentage points
relative to Base, while improving mathematical accuracy from
65.8\% to 66.9\% over LoRA. It also exceeds L2-SP-LoRA and EWC-LoRA
in commonsense accuracy by 5.6 and 7.2 points, respectively,
with higher mathematical accuracy than both.
These results support layer selection as an effective way to balance
adaptation and retention without an additional regularization loss.

\subsection{Code Generation}

Table~\ref{Table:code_generation} shows that LS-LoRA achieves both
the highest mean of the reported pass@$k$ scores and the highest
average commonsense accuracy among the compared fine-tuning methods
on all four models. Specifically, it outperforms all compared
baselines at each reported pass@$k$ on Llama-3.2-1B and Mistral-7B.
Compared with all-layer LoRA, LS-LoRA improves average commonsense
accuracy by 1.2--6.2 percentage points. On Llama-3.2-3B and Llama-3-8B,
its commonsense accuracy exceeds Base, indicating positive transfer
on the evaluated benchmarks rather than reduced forgetting.
On Llama-3.2-1B and Mistral-7B, it reduces the accuracy drop
relative to LoRA but remains below Base.
Overall, LS-LoRA combines competitive code generation performance
with consistently higher average commonsense accuracy.

\subsection{Ablation and Efficiency Analysis}

\paragraph{Similarity-based versus random layer selection.}
Rand-LoRA randomly selects the same number of layers as LS-LoRA,
allowing us to assess the value of the selection criterion beyond
simply reducing the number of adapted layers.
On Llama-3.2-1B (Table~\ref{Table:random_layer}), LS-LoRA achieves
52.2\% average math accuracy and 31.8\% average commonsense accuracy,
compared with 51.4\% and 27.1\% for Rand-LoRA.
Random selection improves retention over all-layer LoRA, but
similarity-based selection yields better averages in both categories.
This advantage holds across all four evaluated models, with results
for the other three reported in Appendix~\ref{app:random_layer}.

\begin{table}[h]\small
\caption{Layer selection on Llama-3.2-1B for mathematical reasoning. Rand-LoRA selects the same number of layers as LS-LoRA at random. Bold marks the best fine-tuned average; higher is better.}
\label{Table:random_layer}
\centering 
\setlength{\tabcolsep}{4.8pt}
{
  \begin{tabular}{llcccccccc}
    \toprule[1.1pt]
    \multirow[c]{2}{*}[-0.5ex]{Model}& \multirow[c]{2}{*}[-0.5ex]{Method} & \multicolumn{4}{c}{Target-Task Performance (Math)} & \multicolumn{4}{c}{General Capability Retention}\\
    \cmidrule(lr){3-6}\cmidrule(lr){7-10}
    &&GSM8K&AQuA&SVAMP&Avg.&ARC-C&OB-QA&SI-QA&Avg. \\
    \midrule
    \multirow{4}{*}{Llama-3.2-1B} 
    & Base       & 31.0 & 6.7  & 60.1 & 32.6 & 28.8 & 30.6 & 40.8 & 33.4 \\
    \addlinespace[0.1em]
    \cdashline{2-10}
    \addlinespace[0.2em]
    & LoRA       & 49.5 & 38.2 & 67.6 & 51.8 & 22.8 & 24.6 & 29.9 & 25.8 \\
    & Rand-LoRA  & 50.8 & 36.6 & 66.8 & 51.4 & 25.3 & 28.0 & 28.1 & 27.1 \\
    \rowcolor{cyan!10} \cellcolor{white} & LS-LoRA   & 50.3 & 37.8 & 68.4 & \textbf{52.2} & 30.0 & 32.2 & 33.1 & \textbf{31.8} \\

\bottomrule[1.1pt]
  \end{tabular}
  }
\end{table}

\paragraph{Memory and training time.}
Table~\ref{Table:Memory_Time} compares LoRA and LS-LoRA on MATH-10K.
Across the four models, LS-LoRA reduces peak GPU memory by
4.2--14.4\% and training time by 10.0--29.6\%.
For Llama-3.2-3B, memory decreases from 13.1 to 11.8~GB and training
time from 1.7 to 1.2~hours, demonstrating practical savings from
restricting adaptation to selected layers.

\begin{table}[h]
\caption{Comparison of peak GPU memory usage and training time
between LoRA and LS-LoRA on the MATH-10K dataset across four language models. Lower is better.}
\label{Table:Memory_Time}
\centering 
\setlength{\tabcolsep}{3pt}
{
  \begin{tabular}{lcccccccc}
    \toprule[1.1pt]
    \multirow[c]{2}{*}[-0.5ex]{Metric}  & \multicolumn{2}{c}{Llama-3.2-1B} & \multicolumn{2}{c}{Llama-3.2-3B} & \multicolumn{2}{c}{Llama-3-8B} & \multicolumn{2}{c}{Mistral-7B}\\
    \cmidrule(lr){2-3}\cmidrule(lr){4-5}\cmidrule(lr){6-7}\cmidrule(lr){8-9}
                   & LoRA&LS-LoRA & LoRA&LS-LoRA & LoRA&LS-LoRA & LoRA&LS-LoRA  \\
    \midrule
    GPU Memory (GB)       & 7.2 & \textbf{6.9}    & 13.1 & \textbf{11.8}  & 24.8 & \textbf{22.5}  & 21.5 & \textbf{18.4}  \\
    Training Time (h)     & 1.0 & \textbf{0.9}    & 1.7  & \textbf{1.2}   & 2.1  & \textbf{1.8}   & 2.7  & \textbf{1.9}   \\
    \bottomrule[1.1pt]
  \end{tabular}
  }
\end{table}

\paragraph{Last-token versus token-averaged similarity.}
Table~\ref{Table:Token_Mean} compares layer-wise cosine similarities
computed using the last token with those averaged over all tokens.
Spearman correlations of 0.92--0.93 across the three Llama models
indicate strong agreement in layer rankings. We therefore use the
last non-padding token as a simple approximation to token averaging.

\paragraph{Calibration sample size.}
On MATH-10K with Llama-3.2-3B, we compare similarity scores computed
from 1, 10, and 100 samples against a 1,000-sample reference
(Table~\ref{Table:Num_Sample}). The corresponding Spearman correlations
are 0.938, 0.990, and 0.999, indicating stable layer rankings even
with small calibration sets. Computing similarity scores with 100 samples
remains inexpensive, taking only 2.8~seconds on Llama-3.2-3B.
We therefore use 100 calibration samples in our experiments to obtain
stable layer rankings at low computational cost.

\begin{table}[h]
\centering
\begin{minipage}[t]{0.48\textwidth}
  \centering
  \caption{Spearman correlation between token-averaged and
last-token cosine similarities across Llama models of different sizes.}
  \label{Table:Token_Mean}
  \setlength{\tabcolsep}{10pt}
  \begin{tabular}{cccc}
    \toprule[1.1pt]
    Model size & 1B & 3B & 8B \\
    \midrule
    $\rho$ & 0.92 & 0.92 & 0.93 \\
    \bottomrule[1.1pt]
  \end{tabular}
\end{minipage}\hfill
\begin{minipage}[t]{0.48\textwidth}
  \centering
  \caption{Spearman correlation between cosine similarities
  estimated using 1, 10, or 100 samples and those estimated
  using 1,000 samples.}
  \label{Table:Num_Sample}
  \setlength{\tabcolsep}{3pt}
  \begin{tabular}{cccc}
    \toprule[1.1pt]
    Number of samples & 1 & 10 & 100 \\
    \midrule
    $\rho$ & 0.938 & 0.990 & 0.999 \\
    \bottomrule[1.1pt]
  \end{tabular}
\end{minipage}
\end{table}

\paragraph{Applicability to other PEFT methods.}
We apply the same layer-selection strategy to DoRA~\citep{liu2024dora} and PiSSA~\citep{NEURIPS2024_db36f4d6} on
Llama-3.2-3B, with detailed results in Appendix~\ref{app:other_peft}
(Table~\ref{Table:dora_pisssa}). LS-DoRA improves average
math accuracy from 66.0\% to 67.0\% and commonsense accuracy from
37.9\% to 58.9\%. LS-PiSSA improves these averages from 66.2\% to
67.6\% and from 59.0\% to 64.7\%, respectively. These results suggest
that the benefits of layer selection extend beyond LoRA.

\section{Conclusion}

We investigated where to apply adaptation to balance target-task learning
and capability retention in LLMs. Our analysis identifies a consistent
negative association between layer-wise input--output cosine similarity
and empirical Fisher sensitivity, motivating a forward-only criterion
for layer selection. Based on this criterion, LS-LoRA places adapters
in selected layers before fine-tuning, without requiring response labels
or backpropagation during selection. Across four evaluated models on
mathematical reasoning and code generation, LS-LoRA achieves higher
average target-task and commonsense reasoning scores than all-layer LoRA,
while reducing memory usage and training time in the mathematical reasoning
experiments. Comparisons with random layer selection support the value of
the selection criterion, and results with DoRA and PiSSA suggest its
applicability beyond LoRA. These findings highlight adaptation location
as a practical design choice for efficient fine-tuning and capability
retention.

\subsection*{AI use statement}

We used generative AI tools to polish the language of the manuscript, 
including improving grammar, clarity, and readability.
We did not use generative AI tools for any other
research or writing tasks. We reviewed all AI-assisted
edits to ensure that they accurately preserved our intended
meaning and technical content. We take full responsibility for the
final content of this work.

\bibliography{references}
\bibliographystyle{iclr2027_conference}

\appendix
\section{Proof of the Local Distributional Sensitivity Proposition}
\label{app:proof_fisher_kl}

\begin{proof}[Proof of Proposition~\ref{pro:fisher_kl}]
Write $g(\delta)=D_{\mathrm{KL}}(p_\theta\Vert p_{\theta+\delta})$.
Differentiating the normalization identity once and twice gives
\[
    \nabla g(0)
    =-\mathbb{E}_{Z\sim p_\theta}[\nabla_\theta\log p_\theta(Z)]=0,
    \qquad
    \nabla^2 g(0)
    =-\mathbb{E}_{Z\sim p_\theta}[\nabla_\theta^2\log p_\theta(Z)]
    =F(\theta).
\]
Since $g(0)=0$, the result follows from the second-order Taylor expansion.
\end{proof}

\section{Cosine Similarity and an Empirical Fisher Lower Bound}
\label{app:proof_cosine_fisher}

We give the conditions and proof for
Theorem~\ref{thm:cosine_fisher_bound}. Both scores use the same
target-task examples $\mathcal{S}_\tau=\{(x_i,y_i)\}_{i=1}^N$,
with $N_c=N$ and $\mathcal{V}=\{(i,T_i)\}_{i=1}^N$.
For the analysis, set $\epsilon=0$ in
Eq.~\ref{eq:layer_cosine_similarity}.

\paragraph{Full-sequence representations.}
Let $\ell_i=\log p_\theta(y_i\mid x_i)$. In this appendix,
$H_i^{(l-1)}$ and $H_i^{(l)}$ denote the layer representations
at all positions used to compute this full-response log-likelihood
under teacher forcing. The index $T_i$ still denotes the final valid
input position, not the end of the response.
For a deterministic causal model with the same positions and mask,
the representation at $T_i$ agrees with that from the input-only
forward pass used to compute $s_l$.
Define the full residual, its Jacobian, and the output gradient by
\begin{equation}
\begin{aligned}
    \boldsymbol{\Delta}_i^{(l)}
    &=\operatorname{vec}\bigl(H_i^{(l)}-H_i^{(l-1)}\bigr),\\
    J_i^{(l)}
    &=\frac{\partial\boldsymbol{\Delta}_i^{(l)}}
    {\partial\theta^{(l)}},
    \qquad
    g_i^{(l)}=\nabla_{\operatorname{vec}(H_i^{(l)})}\ell_i.
\end{aligned}
\label{eq:cosine_fisher_residual}
\end{equation}
Assume differentiability, $\theta^{(l)}\neq0$, that the layer input is
independent of $\theta^{(l)}$, and that all dependence of $\ell_i$ on
these parameters passes through the full output $H_i^{(l)}$.
This allows dependence on every output token, as in a standard
feed-forward Transformer with parameters not shared across layers.
The chain rule then gives
$\nabla_{\theta^{(l)}}\ell_i=(J_i^{(l)})^\top g_i^{(l)}$.

\paragraph{Nonzero selected representations.}
Assume that the selected input and output vectors are nonzero and set
\begin{equation}
    m_l=\min_{1\leq i\leq N}
    \|h_{i,T_i}^{(l-1)}\|_2\|h_{i,T_i}^{(l)}\|_2>0.
    \label{eq:cosine_fisher_norm_product}
\end{equation}
Their norms may differ both within and across examples.

\paragraph{Average directional task sensitivity.}
Assume that there exists $\kappa_l>0$ such that
\begin{equation}
    \sum_{i=1}^N
    |\langle g_i^{(l)},\boldsymbol{\Delta}_i^{(l)}\rangle|^2
    \geq\kappa_l^2\sum_{i=1}^N
    \|\boldsymbol{\Delta}_i^{(l)}\|_2^2.
    \label{eq:cosine_fisher_alignment}
\end{equation}
This condition requires task sensitivity along the residuals on average.
Individual examples may have zero directional sensitivity; no
positive lower bound is required for each example separately.

\paragraph{Approximate local parameter direction.}
Assume that there exist a common direction $v_l\in\mathbb{R}^{d_l}$,
$C_l>0$, and $0\leq\eta_l<1$ such that
\begin{equation}
    \|v_l\|_2\leq C_l\|\theta^{(l)}\|_2,
    \qquad
    e_i^{(l)}=J_i^{(l)}v_l-\boldsymbol{\Delta}_i^{(l)},
    \label{eq:cosine_fisher_local_direction}
\end{equation}
and
\begin{equation}
    \left(\sum_{i=1}^N|\langle g_i^{(l)},e_i^{(l)}\rangle|^2\right)^{1/2}
    \leq\eta_l
    \left(\sum_{i=1}^N
    |\langle g_i^{(l)},\boldsymbol{\Delta}_i^{(l)}\rangle|^2\right)^{1/2}.
    \label{eq:cosine_fisher_direction_error}
\end{equation}
The direction need only reproduce the residual approximately in its
effect on the task. Errors orthogonal to the output gradients are
unrestricted. Exact reproduction, $J_i^{(l)}v_l=\boldsymbol{\Delta}_i^{(l)}$,
is a special case with $\eta_l=0$.

For example, a residual sublayer of the form
$r_l(H;W,\phi)=a_\phi(H)W^\top$ satisfies the exact local condition with
$v_l=(\operatorname{vec}(W),0)$ and $C_l=1$, when
$\theta^{(l)}=(\operatorname{vec}(W),\phi)$.
The map $a_\phi$ may be nonlinear. This example establishes only
the local condition, not the average task-sensitivity condition.

\begin{proof}[Proof of Theorem~\ref{thm:cosine_fisher_bound}]
Write $u_i=h_{i,T_i}^{(l-1)}$, $w_i=h_{i,T_i}^{(l)}$, and
$c_i=\langle u_i,w_i\rangle/(\|u_i\|_2\|w_i\|_2)$.
Since the full residual contains the selected token residual,
\begin{equation}
\begin{aligned}
    \|\boldsymbol{\Delta}_i^{(l)}\|_2^2
    &\geq\|w_i-u_i\|_2^2\\
    &=(\|w_i\|_2-\|u_i\|_2)^2
      +2\|u_i\|_2\|w_i\|_2(1-c_i)\\
    &\geq2m_l(1-c_i).
\end{aligned}
\label{eq:cosine_fisher_geometry}
\end{equation}
Averaging gives
$N^{-1}\sum_i\|\boldsymbol{\Delta}_i^{(l)}\|_2^2
\geq2m_l(1-s_l)$.

For the parameter gradient, the chain rule gives
\[
    \langle v_l,\nabla_{\theta^{(l)}}\ell_i\rangle
    =\langle g_i^{(l)},\boldsymbol{\Delta}_i^{(l)}\rangle
     +\langle g_i^{(l)},e_i^{(l)}\rangle.
\]
Apply the reverse triangle inequality to the vectors of these
scalar quantities over the $N$ examples, then use
Eqs.~\ref{eq:cosine_fisher_direction_error}
and~\ref{eq:cosine_fisher_alignment}:
\begin{equation}
\begin{aligned}
    \left(\sum_i
    |\langle v_l,\nabla_{\theta^{(l)}}\ell_i\rangle|^2\right)^{1/2}
    &\geq(1-\eta_l)
    \left(\sum_i
    |\langle g_i^{(l)},\boldsymbol{\Delta}_i^{(l)}\rangle|^2\right)^{1/2}\\
    &\geq(1-\eta_l)\kappa_l
    \left(\sum_i\|\boldsymbol{\Delta}_i^{(l)}\|_2^2\right)^{1/2}.
\end{aligned}
\end{equation}
On the other hand, Cauchy--Schwarz implies
\[
    \sum_i|\langle v_l,\nabla_{\theta^{(l)}}\ell_i\rangle|^2
    \leq C_l^2\|\theta^{(l)}\|_2^2
    \sum_i\|\nabla_{\theta^{(l)}}\ell_i\|_2^2.
\]
Combining these inequalities and dividing by $Nd_l$ yields
\[
    I_l
    \geq\frac{(1-\eta_l)^2\kappa_l^2}
    {Nd_l C_l^2\|\theta^{(l)}\|_2^2}
    \sum_i\|\boldsymbol{\Delta}_i^{(l)}\|_2^2
    \geq\frac{2(1-\eta_l)^2\kappa_l^2m_l}
    {d_l C_l^2\|\theta^{(l)}\|_2^2}(1-s_l).
\]
\end{proof}

\begin{table}[t]\small
\caption{Comparison of layer selection strategies for LoRA fine-tuning
on mathematical reasoning tasks.
LoRA adapts all layers; LS-LoRA selects layers with low input–output
cosine similarity; Rand-LoRA randomly selects the same number of layers
as LS-LoRA.
Bold marks the best average among fine-tuned methods for each model.
Higher is better.}
\label{Table:random_layer_additional}
\centering 
\setlength{\tabcolsep}{4.8pt}
{
  \begin{tabular}{llcccccccc}
    \toprule[1.1pt]
    \multirow[c]{2}{*}[-0.5ex]{Model}& \multirow[c]{2}{*}[-0.5ex]{Method} & \multicolumn{4}{c}{Target-Task Performance (Math)} & \multicolumn{4}{c}{General Capability Retention}\\
    \cmidrule(lr){3-6}\cmidrule(lr){7-10}
    &&GSM8K&AQuA&SVAMP&Avg.&ARC-C&OB-QA&SI-QA&Avg. \\
    \midrule
    \multirow{4}{*}{Llama-3.2-3B} 
    & Base   & 62.0 & 19.3 & 83.7 & 55.0 & 62.2 & 56.4 & 58.8 & 59.1 \\
    \addlinespace[0.1em]
    \cdashline{2-10}
    \addlinespace[0.2em]
    & LoRA       & 69.6 & 47.6 & 80.3 & 65.8 & 44.1 & 52.4 & 34.9 & 43.8 \\
    & Rand-LoRA  & 69.6 & 46.1 & 80.2 & 65.3 & 57.3 & 59.0 & 53.8 & 56.7 \\
    \rowcolor{cyan!10} \cellcolor{white} &LS-LoRA   & 69.5 & 48.8 & 82.3 & \textbf{66.9} & 59.3 & 62.4 & 52.6 & \textbf{58.1} \\

    \midrule
    \multirow{4}{*}{Llama-3-8B} 
    & Base       & 59.6 & 49.6 & 53.8 & 54.3 & 33.1 & 25.6 & 37.4 & 32.0 \\
    \addlinespace[0.1em]
    \cdashline{2-10}
    \addlinespace[0.2em]
    & LoRA       & 69.8 & 42.6 & 79.6 & 64.0 & 25.8 & 31.8 & 18.5 & 25.4 \\
    & Rand-LoRA  & 70.1 & 44.5 & 78.4 & 64.3 & 29.0 & 32.0 & 26.2 & 29.1 \\
    \rowcolor{cyan!10} \cellcolor{white} & LS-LoRA   & 70.4 & 45.3 & 78.6 & \textbf{64.8} & 33.9 & 37.2 & 33.4 & \textbf{34.8} \\

    \midrule
    \multirow{4}{*}{Mistral-7B} 
    & Base   & 38.5 & 20.5 & 55.2 & 38.1 & 66.6 & 61.4 & 65.1 & 64.4 \\
    \addlinespace[0.1em]
    \cdashline{2-10}
    \addlinespace[0.2em]
    & LoRA       & 59.8 & 32.7 & 73.8 & 55.4 & 34.6 & 39.0 & 37.5 & 37.0 \\
    & Rand-LoRA  & 61.7 & 31.5 & 74.0 & 55.7 & 41.3 & 39.0 & 35.2 & 38.5 \\
    \rowcolor{cyan!10} \cellcolor{white} & LS-LoRA   & 61.0 & 33.1 & 74.6 & \textbf{56.2} & 48.3 & 42.4 & 47.9 & \textbf{46.2} \\

\bottomrule[1.1pt]
  \end{tabular}
  }
\end{table}

\begin{table}[h]\small
\caption{Applying our layer selection strategy to DoRA and PiSSA on
Llama-3.2-3B for mathematical reasoning.
LS-DoRA and LS-PiSSA apply the respective adapters only to layers with
low input--output cosine similarity.
We report accuracy (\%) on mathematical reasoning and commonsense reasoning
benchmarks, with the latter assessing general capability retention.
Avg. denotes the mean within each category; \textbf{bold} marks the better
average within each method pair. Higher is better.}
\label{Table:dora_pisssa}
\centering 
\setlength{\tabcolsep}{9pt}
{
  \begin{tabular}{lcccccccc}
    \toprule[1.1pt]
    \multirow[c]{2}{*}[-0.5ex]{Method}& \multicolumn{4}{c}{Target-Task Performance (Math)} & \multicolumn{4}{c}{General Capability Retention}\\
    \cmidrule(lr){2-5}\cmidrule(lr){6-9}
    &GSM8K&AQuA&SVAMP&Avg.&ARC-C&OB-QA&SI-QA&Avg. \\
    \midrule
    DoRA       & 68.1 & 48.8 & 81.2 & 66.0 & 36.4 & 43.8 & 33.4 & 37.9 \\
    LS-DoRA    & 69.5 & 49.2 & 82.2 & \textbf{67.0} & 60.3 & 62.8 & 53.7 & \textbf{58.9} \\
    \midrule
    PiSSA      & 70.6 & 48.0 & 79.9 & 66.2 & 59.2 & 61.8 & 55.9 & 59.0 \\
    LS-PiSSA   & 71.3 & 50.0 & 81.4 & \textbf{67.6} & 68.6 & 66.0 & 59.5 & \textbf{64.7} \\

\bottomrule[1.1pt]
  \end{tabular}
  }
\end{table}

\section{Spearman's Rank Correlation}
\label{app:spearman}

\begin{definition}[Spearman's rank correlation]
\label{lem:spearman}
Let $\mathbf{a}=(a_1,\ldots,a_n)$ and $\mathbf{b}=(b_1,\ldots,b_n)$
be two nonconstant vectors of paired observations, with $n\geq2$.
Let $r_i$ and $q_i$ be the ascending ranks of $a_i$ and $b_i$
within their respective vectors, assigning average ranks to ties.
Their Spearman rank correlation $\rho$ is the Pearson correlation of these ranks:
\begin{equation}
    \rho(\mathbf{a},\mathbf{b})
    =\operatorname{Corr}(\mathbf{r},\mathbf{q})
    =\frac{\sum_{i=1}^{n}(r_i-\bar r)(q_i-\bar q)}
    {\sqrt{\sum_{i=1}^{n}(r_i-\bar r)^2}
     \sqrt{\sum_{i=1}^{n}(q_i-\bar q)^2}},
    \label{eq:spearman_general}
\end{equation}
where $\bar r$ and $\bar q$ are the mean ranks and $-1\leq\rho\leq1$.
\end{definition}

When Spearman's rank correlation $\rho$ is close to $1$, the two variables tend to increase
together. When $\rho$ is close to $-1$, one tends to decrease
as the other increases.

For each model and task $\tau$, we apply Definition~\ref{lem:spearman}
to the empirical Fisher scores $\mathbf{I}=(I_1,\ldots,I_L)$ and
input–output cosine similarities $\mathbf{s}=(s_1,\ldots,s_L)$,
treating each layer as one paired observation:
\begin{equation}
    \rho_\tau
    = \operatorname{Corr}
    \left(
        \operatorname{rank}(\mathbf{I}),
        \operatorname{rank}(\mathbf{s})
    \right).
    \label{eq:fisher_similarity_correlation}
\end{equation}
where $\operatorname{rank}$ uses the ranking convention in
Definition~\ref{lem:spearman}.

\section{Additional Random Layer Selection Results}
\label{app:random_layer}

Table~\ref{Table:random_layer_additional} extends the comparison in
Table~\ref{Table:random_layer} to Llama-3.2-3B, Llama-3-8B, and Mistral-7B.
Rand-LoRA selects the same number of layers as LS-LoRA at random.
Across all three models, LS-LoRA achieves higher average mathematical
reasoning and commonsense reasoning accuracy than Rand-LoRA,
supporting the value of similarity-based selection beyond simply
reducing the number of adapted layers.

\section{Applicability to Other PEFT Methods}
\label{app:other_peft}

Table~\ref{Table:dora_pisssa} reports results for DoRA and PiSSA and
their layer-selective variants on Llama-3.2-3B. Applying our
input--output similarity criterion improves average accuracy on both
mathematical reasoning and commonsense reasoning benchmarks for both methods.

\end{document}

%% file: math_commands.tex
\usepackage{amsmath,amsfonts,bm}

\def\eqref#1{equation~\ref{#1}}

\def\1{\bm{1}}

\DeclareMathAlphabet{\mathsfit}{\encodingdefault}{\sfdefault}{m}{sl}
\SetMathAlphabet{\mathsfit}{bold}{\encodingdefault}{\sfdefault}{bx}{n}

